\documentclass[letterpaper, 10 pt, conference]{ieeeconf}

\IEEEoverridecommandlockouts
\usepackage{amsmath}
\usepackage{amstext}
\usepackage{amssymb}
\usepackage{amsfonts}
\usepackage{theorem}
\usepackage{graphicx}
\usepackage[hang]{subfigure}
\usepackage{psfrag}
\usepackage{float}
\usepackage{floatflt}
\usepackage{graphicx}
\usepackage{calrsfs}
\usepackage{mathrsfs}

\newcommand{\be}[1]{\begin{equation}\label{#1}}
\newcommand{\benon}{\begin{equation*}}  
\newcommand{\bemuln}[1]{\begin{multline}\label{#1}}
\newcommand{\bemul}{\begin{multline*}}
\newcommand{\bee}{\begin{eqnarray*}}
\newcommand{\eee}{\end{eqnarray*}}
\newcommand{\been}[1]{\begin{eqnarray}\label{#1}}
\newcommand{\eeen}{\end{eqnarray}}
\newcommand{\began}[1]{\begin{gather}\label{#1}}
\newcommand{\bega}{\begin{gather*}}
\newcommand{\bealn}[1]{\begin{align}\label{#1}}
\newcommand{\beal}{\begin{align*}}
\newcommand{\bealatn}[2]{\begin{alignat}{#1}\label{#2}}
\newcommand{\bealat}{\begin{alignat*}}
\newcommand{\bexalatn}[1]{\begin{xalignat}\label{#1}}
\newcommand{\bexalat}{\begin{xalignat*}}
\newcommand{\mbb}{\mathbb}

{\theoremstyle{plain} \newtheorem{thm}{Theorem}[section]

\newtheorem{lemma}[thm]{Lemma}
}
{\theoremstyle{break} \theorembodyfont{\it}

\newtheorem{ass}{Assumption}
 
\newtheorem{cond}{Condition} }

\def\ba{{\mathbf a}}
\def\bb{{\mathbf b}}

\def\bbf{{\mathbf f}}

\def\bh{{\mathbf h}}
\def\bi{{\mathbf i}}
\def\bj{{\mathbf j}}

\def\br{{\mathbf r}}
\def\bs{{\mathbf s}}

\def\bv{{\mathbf v}}

\def\bx{{\mathbf x}}  
\def\by{{\mathbf y}}
\def\bz{{\mathbf z}}
\def\bA{{\mathbf A}}

\def\bF{{\mathbf F}}
\def\bG{{\mathbf G}}

\def\bI{{\mathbf I}}

\def\bQ{{\mathbf Q}}

\def\texitem#1{\par\smallskip\noindent\hangindent 25pt
               \hbox to 25pt {\hss #1 ~}\ignorespaces}

\newcommand{\bzero}{{\mathbf{0}}}

\newcommand{\scrS}{\mathcal{S}}

\newcommand{\btheta}{\boldsymbol{\theta}}

\newcommand{\bmu}{\boldsymbol{\mu}}

\newcommand{\bphi}{{\boldsymbol{\phi}}}

\newcommand{\bXi}{\boldsymbol{\Xi}}
\newcommand{\bpsi}{\boldsymbol{\psi}}

\newcommand{\qed}{\newline \mbox{ } \hfill 
            \rule[-1pt]{2.1mm}{2.1mm}\par\vskip 10pt }

\begin{document}

\title{Least Squares Temporal Difference Actor-Critic Methods with Applications to Robot Motion Control~\authorrefmark{1} \thanks{* Research partially supported by the NSF under grant EFRI-0735974, by the DOE under grant DE-FG52-06NA27490, by the ODDR\&E MURI10 program under grant N00014-10-1-0952, and by ONR MURI under grant N00014-09-1051.}}

\author{
Reza Moazzez Estanjini\authorrefmark{2}, 
Xu Chu Ding\authorrefmark{3}, 
Morteza Lahijanian\authorrefmark{3}, Jing Wang\authorrefmark{2},\\
Calin A. Belta\authorrefmark{3}, 
and Ioannis Ch. Paschalidis\authorrefmark{4}
\thanks{$\dagger$ Reza Moazzez Estanjini and Jing Wang are with the Division of
Systems Eng., Boston University, 8 St. Mary's St., Boston, MA 02215,
email: {\tt \{reza,wangjing\}@bu.edu}.}
\thanks{$\ddagger$ Xu Chu Ding, Morteza Lahijanian, and Calin A. Belta
are with the Dept. of  
Mechanical Eng., Boston University, 15 St. Mary's St., Boston, MA
02215, email:  
{\tt \{xcding,morteza,cbelta\}@bu.edu}.}
\thanks{$\S$ Ioannis Ch. Paschalidis is 
  with the Dept. of Electrical \&  
Computer Eng., and the Division of
Systems Eng., Boston University, 8 St. Mary's St., Boston, MA 02215,
email: {\tt yannisp@bu.edu}.}
\thanks{$\S$ Corresponding author}
}

\maketitle

\begin{abstract}
We consider the problem of finding a control policy for a Markov
Decision Process (MDP) to maximize the probability of reaching some
states while avoiding some other states.  This problem is motivated by
applications in robotics, where such problems naturally arise when
probabilistic models of robot motion are required to satisfy temporal
logic task specifications.  We transform this problem into a Stochastic
Shortest Path (SSP) problem and develop a new approximate dynamic
programming algorithm to solve it. This algorithm is of the actor-critic
type and uses a least-square temporal difference learning method. It
operates on sample paths of the system and optimizes the policy within a
pre-specified class parameterized by a parsimonious set of
parameters. We show its convergence to a policy corresponding to a
stationary point in the parameters' space.  Simulation results confirm
the effectiveness of the proposed solution.
\end{abstract}

\begin{keywords}
Markov Decision Processes, dynamic programming, actor-critic methods,
robot motion control, robotics.
\end{keywords}

{\allowdisplaybreaks
\section{Introduction}\label{sec:intro}

Markov Decision Processes (MDPs) have been widely used in a variety of
application domains.  In particular, they have been increasingly used to
model and control autonomous agents subject to noises in their sensing
and actuation, or uncertainty in the environment they operate.  Examples
include: unmanned aircraft \cite{temizercollision}, ground robots
\cite{LaWaAnBe-ICRA10}, and steering of medical needles
\cite{alterovitz2007stochastic}.  In these studies, the underlying
motion of the system cannot be predicted with certainty, but they can be
obtained from the sensing and the actuation model through a simulator or
empirical trials, providing transition probabilities.

Recently, the problem of controlling an MDP from a temporal logic
specification has received a lot of attention. Temporal logics such as
Linear Temporal Logic (LTL) and Computational Tree Logic (CTL) are
appealing as they provide formal, high level languages in which the
behavior of the system can be specified (see
\cite{baier2008principles}).  In the context of MDPs, providing
probabilistic guarantees means finding optimal policies that maximize
the probabilities of satisfying these specifications.  In
\cite{LaWaAnBe-ICRA10,Ding11}, it has been shown that, the problem of
finding an optimal policy that maximizes the probability of satisfying a
temporal logic formula can be naturally translated to one of maximizing
the probability of reaching a set of states in the MDP.  Such problems
are referred to as Maximal Reachability Probability (MRP) problems.  It
has been known \cite{alterovitz2007stochastic} that they are equivalent
to Stochastic Shortest Path (SSP) problems, which belong to a standard
class of infinite horizon problems in dynamic programming.

However, as suggested in \cite{LaWaAnBe-ICRA10,Ding11}, these problems usually involve MDPs with large state spaces.  For example, in order to synthesize an optimal policy for an MDP satisfying an LTL formula, one needs to solve an MRP problem on a much larger MDP, which is the product of the original MDP and an automaton representing the formula.  %This automata can be, in the worst case, double exponential with respect to the length of the formula.    
Thus, computing the exact solution can be computationally prohibitive for realistically-sized settings.  Moreover, in some cases, the system of interest is so complex that it is not feasible to determine transition probabilities for all actions and states explicitly.  

Motivated by these limitations, in this paper we develop a new approximate dynamic programming algorithm to solve SSP MDPs and we establish its convergence.  The algorithm is of the actor-critic type and uses a \emph{Least Square Temporal Difference} (LSTD) learning method.  Our proposed algorithm is based on sample paths, and thus only requires transition probabilities along the sampled paths and not over the entire state space.  

Actor-critic algorithms are typically used to optimize some {\em Randomized Stationary Policy} (RSP) using policy gradient estimation. RSPs are parameterized by a parsimonious set of parameters and the objective is to optimize the policy with respect to these parameters. To this end, one needs to estimate appropriate policy gradients, which can be done using learning methods that are much more efficient than computing a cost-to-go function over the entire state-action space. Many different versions of actor-critic algorithms have been proposed which have been shown to be quite efficient for various applications (e.g., in robotics \cite{e4} and navigation \cite{e1}, power management of wireless transmitters \cite{e2}, biology \cite{e5}, and optimal bidding for electricity generation companies \cite{e3}). 

A particularly attractive design of the actor-critic architecture was proposed in \cite{Konda03}, where the \emph{critic} estimates the policy gradient using sequential observations from a sample path while the \emph{actor} updates the policy at the same time, although at a slower time-scale. It was proved that the estimate of the critic tracks the slowly varying policy asymptotically under suitable conditions. A center piece of these conditions is a relationship between the actor step-size and the critic step-size, which will be discussed later.

The critic of \cite{Konda03} uses first-order variants of the {\em Temporal Difference} (TD) algorithm (TD(1) and
TD($\lambda$)). However, it has been shown that the least squares methods -- LSTD (Least Squares TD) and LSPE (Least Squares Policy Evaluation) -- are superior in terms of convergence rate (see \cite{Bradtke96,Konda}). LSTD and LSPE were first proposed for discounted cost problems in \cite{Bradtke96} and \cite{BI96}, respectively. Later, \cite{Konda} showed that the convergence rate of LSTD is optimal. Their results clearly demonstrated that LSTD converges much faster and more reliably than TD(1) and TD($\lambda$). 

Motivated by these findings, we propose an actor-critic algorithm that adopts LSTD learning methods tailored to SSP problems, while at the same time maintains the concurrent update architecture of the actor and the critic. (Note that \cite{Peters08} also used LSTD in an actor-critic method, but the actor had to wait for the critic to converge before making each policy update.) To illustrate salient features of the approach, we present a case study where a robot in a large environment is required to satisfy a task specification of ``go to a set of goal states while avoiding a set of unsafe states.'' (We note that more complex task specifications can be directly converted to MRP problems as shown in \cite{LaWaAnBe-ICRA10,Ding11}.) 

The rest of the paper is organized as follows.  We formulate the problem in Sec. \ref{S:formulation}. The LSTD actor-critic algorithm with concurrent updates is presented in Sec. \ref{S:lstdac}, where the convergence of the algorithm is shown. 
A case study is presented in Sec. \ref{S:case}.  We conclude the paper in Sec. \ref{S:conclusion}.

\paragraph*{\bf Notation} We use bold letters to denote vectors and matrices; typically vectors are lower case and matrices upper case. Vectors are assumed to be column vectors unless explicitly stated otherwise. Transpose is denoted by prime. For any $m\times n$ matrix $\bA$, with rows $\ba_1,\ldots,\ba_m\in \mbb{R}^n$, $\bv(\bA)$ denotes the column vector $(\ba_1,\ldots,\ba_m)$. $\|\cdot\|$ stands for the Euclidean norm and $\|\cdot\|_{\btheta}$ is a special norm in the MDP state-action space that we will define later. $\bzero$ denotes a vector or matrix with all components set to zero and $\bI$ is the identity matrix. $|S|$ denotes the cardinality of a set $S$.

\section{Problem Formulation}\label{S:formulation}
%The problem considered here is rather standard, and we will state it briefly.

%\subsection{SSP MDP}\label{S:SSPMDP}
Consider an SSP MDP with finite state and action spaces. Let $k$ denote time, $\mathbb{X}$ denote the state space with cardinality $|\mbb{X}|$, and $\mathbb{U}$ denote the action space with cardinality $|\mbb{U}|$. Let $\bx_k\in \mbb{X}$ and $u_k\in \mbb{U}$ be the state of the system and the action taken at time $k$, respectively. Let $g(\bx_k,u_k)$ be the one-step cost of applying action $u_k$ while the system is at state $\bx_k$. Let $\bx_0$ and $\bx^*$ denote the initial state and the special cost-free termination state, respectively. Let $p(\bj|\bx_k,u_k)$ denote the state transition probabilities (which are typically not explicitly known); that is, $p(\bj|\bx_k,u_k)$ is the probability of transition from state $\bx_k$ to state $\bj$ given that action $u_k$ is taken while the system is at state $\bx_k$. A policy $\bmu$ is said to be {\em proper} if, when using this policy, there is a positive probability that $\bx^*$ will be reached after at most $|\mbb{X}|$ transitions, regardless of the initial state $\bx_0$. We make the following assumption.

\begin{ass}\label{C:FirstProperAss} There exist a proper stationary policy. 
\end{ass}

The policy candidates are assumed to belong to a parameterized family of {\it Randomized Stationary Policies} (RSPs) $\{\mu_{\btheta}(u|\bx)\mid \btheta\in\mathbb{R}^n\}$. That is, given a state $\bx\in \mbb{X}$ and a parameter $\btheta$, the policy applies action $u\in \mbb{U}$ with probability $\mu_{\btheta} (u|\bx)$. Define the {\em expected total cost} $\bar{\alpha}(\btheta)$ to be $\lim_{t\to\infty} E\{\sum_{k=0}^{t-1} g(\bx_k,u_k)|\bx_0\}$ where $u_k$ is generated according to RSP $\mu_{\btheta} (u|\bx)$. The goal is to optimize the expected total cost $\bar{\alpha}(\btheta)$ over the $n$-dimensional vector $\btheta$. 

%\subsection{Actor-Critic Algorithms}
With no explicit model of the state transitions but only a sample path denoted by $\{\bx_k,u_k\}$, the actor-critic algorithms typically optimize $\btheta$ locally in the following way: first, the critic estimates the policy gradient $\nabla\bar{\alpha}(\btheta)$ using a {\em Temporal Difference (TD)} algorithm; then the actor
modifies the policy parameter along the gradient direction. Let the operator $P_{\btheta}$
denote taking expectation after one transition. More precisely, for a function $f(\bx,u)$, $(P_{\btheta}
f)(\bx,u)=\sum_{\bj\in \mathbb{X},\nu\in \mathbb{U}}\mu_{\btheta}(\nu|\bj)p(\bj|\bx,u)f(\bj,\nu)$. Define the $Q_{\btheta}$-value function to be any function satisfying the Poisson equation
$$Q_{\btheta}(\bx,u)=g(\bx,u)+(P_{\btheta}
Q_{\btheta})(\bx,u),$$ 
where $Q_{\btheta}(\bx,u)$ can be interpreted as the expected future cost we incur if we start at state $\bx$,
apply control $u$, and then apply RSP $\mu_{\btheta}$. We note that in general, the Poisson equation need not hold for SSP, however, it holds if the policy corresponding to RSP $\mu_{\btheta}$ is a proper policy \cite{Bertsekas-Book}. We make the following assumption.

\begin{ass}\label{C:SecondProperAss} For any $\btheta$, and for any $\bx\in\mathbb{X}$, $\mu_{\btheta}(u|\bx)>0$ if action $u$ is feasible at state $\bx$, and $\mu_{\btheta}(u|\bx)\equiv 0$ otherwise.
\end{ass}

We note that one possible RSP for which Assumption \ref{C:SecondProperAss} holds is the ``Boltzmann'' policy (see \cite{Sutton98}), that is
\begin{equation}
\label{eq:boltzmann}
\mu_{\btheta}(u|\bx)=\frac{\exp(h^{(u)}_{\btheta}(\bx))}{\sum_{a\in \mathbb{U}} \exp(h^{(a)}_{\btheta}(\bx))},
 \end{equation}
where $h^{(u)}_{\btheta}(\bx)$ is a function that corresponds to action $u$ and is parameterized by $\btheta$. The Boltzmann policy is simple to use and is the policy that will be used in the case study in Sec. \ref{S:case}.

\begin{lemma}\label{T:mainThmm} If Assumptions \ref{C:FirstProperAss} and \ref{C:SecondProperAss} hold, then for any $\btheta$ the policy corresponding to RSP $\mu_{\btheta}$ is proper.
\end{lemma}

\begin{proof} The proof follows from the definition of a proper policy.
\end{proof}

Under suitable ergodicity conditions, $\{\bx_k\}$ and $\{\bx_k,u_k\}$ are Markov chains with stationary distributions under a fixed policy. These stationary distributions are denoted by $\pi_{\btheta}(\bx)$ and $\eta_{\btheta}(\bx,u)$, respectively. We will not elaborate on the ergodicity conditions, except to note that it suffices that the process $\{\bx_k\}$ is irreducible and aperiodic given any $\btheta$, and Assumption \ref{C:SecondProperAss} holds. Denote by $\bQ_{\btheta}$ the $(|\mbb{X}| |\mbb{U}|)$-dimensional vector $\bQ_{\btheta}=(Q_{\btheta}(\bx,u);\ \forall \bx\in\mbb{X},\ u\in \mbb{U})$. Let now $$\bpsi_{\btheta}(\bx,u)=\nabla_{\btheta} \ln\mu_{\btheta}(u|\bx),$$ where
$\bpsi_{\btheta}(\bx,u) =\bzero$ when $\bx,u$ are such that $\mu_{\btheta}(u|\bx)\equiv 0$ for all $\btheta$. It is assumed that $\bpsi_{\btheta}(\bx,u)$ is bounded and continuously differentiable. We write
$\bpsi_{\btheta}(\bx,u)=(\psi_{\btheta}^1(\bx,u),\ldots,\psi_{\btheta}^n(\bx,u))$ where $n$ is the dimensionality of $\btheta$. As we did in defining $\bQ_{\btheta}$ we will denote by $\bpsi_{\btheta}^i$ the $(|\mbb{X}| |\mbb{U}|)$-dimensional vector $\bpsi_{\btheta}^i=(\psi_{\btheta}^i(\bx,u);\ \forall \bx\in\mbb{X},\ u\in \mbb{U})$.

A key fact underlying the actor-critic algorithm is that the policy gradient can be expressed as (Theorem 2.15 in \cite{Konda})
\[
\frac{\partial\bar{\alpha}(\btheta)}{\partial\theta_i}=\langle
\bQ_{\btheta},\bpsi_{\btheta}^i\rangle_{\btheta}, \ i=1,\ldots,n,
\]
where for any two functions $f_1$ and $f_2$ of $\bx$ and $u$, expressed
as $(|\mbb{X}| |\mbb{U}|)$-dimensional vectors $\bbf_1$ and $\bbf_2$, we define
\begin{equation}\label{E:innerproduct}
\langle \bbf_1,\bbf_2\rangle_{\btheta}
\stackrel{\triangle}{=}\sum_{\bx,u}\eta_{\btheta}(\bx,u) f_1(\bx,u) f_2(\bx,u).
\end{equation}
Let $\|\cdot\|_{\btheta}$ denote the norm induced by the inner product
(\ref{E:innerproduct}), {\em i.e.}, $\|\bbf\|_{\btheta}^2=\langle
\bbf,\bbf\rangle_{\btheta}$.  Let also $\scrS_{\btheta}$ be the subspace of
$\mbb{R}^{|\mbb{X}| |\mbb{U}|}$ spanned by the vectors
$\bpsi_{\btheta}^i$, $i=1,\ldots,n$ and denote by $\Pi_{\btheta}$ the
projection with respect to the norm $\|\cdot\|_{\btheta}$ onto
$\scrS_{\btheta}$, {\em i.e.}, for any $\bbf\in \mbb{R}^{|\mbb{X}| |\mbb{U}|}$,
$\Pi_{\btheta} \bbf$ is the unique vector in $\scrS_{\btheta}$ that minimizes
$\|\bbf-\hat{\bbf}\|_{\btheta}$ over all $\hat{\bbf}\in
\scrS_{\btheta}$. Since for all $i$
$$\langle \bQ_{\btheta},\bpsi_{\btheta}^i\rangle_{\btheta}=\langle \Pi_{\btheta}\bQ_{\btheta},\bpsi_{\btheta}^i\rangle_{\btheta},
$$
it is sufficient to know the projection of $\bQ_{\btheta}$ onto
$\scrS_{\btheta}$ in order to compute $\nabla \bar{\alpha}(\btheta)$. One
possibility is to approximate $\bQ_{\btheta}$ with a parametric linear
architecture of the following form (see \cite{Konda03}): 
\begin{equation}\label{E:proj}
Q_{\btheta}^\br(\bx,u)=\bpsi_{\btheta}'(\bx,u)\br^*,\qquad \br^*\in\mathbb{R}^n.
\end{equation}
This dramatically reduces the complexity of learning from the space
$\mathbb{R}^{|\mathbb{X}| |\mathbb{U}|}$ to the space
$\mathbb{R}^n$. Furthermore, the temporal difference algorithms can be
used to learn such an $\br^*$ effectively. The elements of $\bpsi_{\btheta}(\bx,u)$ are understood as features associated with an $(\bx,u)$ state-action pair in the sense of basis functions used to develop an
approximation of the $Q_{\btheta}$-value function.

\section{Actor-Critic Algorithm Using LSTD}
\label{sec:actorCriticAlgo}

The critic in \cite{Konda03} used either TD($\lambda$) or TD(1). The algorithm we propose uses least squares TD methods (LSTD in particular) instead as they have been shown to provide far superior performance. In the sequel, we first describe the LSTD actor-critic algorithm and then we prove its convergence.

\subsection{The Algorithm}

The algorithm uses a sequence of simulated trajectories, each of which starting at a given $\bx_0$ and ending as soon as $\bx^*$ is visited for the first time in the sequence. Once a trajectory is completed, the state of the system is reset to the initial state $\bx_0$ and the process is repeated. 

Let $\bx_k$ denote the state of the system at time $k$. Let $\br_k$, the iterate for $\br^*$ in (\ref{E:proj}), be the parameter vector of the critic at time $k$, $\btheta_k$ be the parameter vector of the actor at time $k$, and $\bx_{k+1}$ be the new state, obtained after action $u_k$ is applied when the state is $\bx_k$. A new action
$u_{k+1}$ is generated according to the RSP corresponding to the actor parameter $\btheta_k$ (see \cite{Konda03}). The critic and the actor carry out the following updates, where $\bz_k\in \mathbb{R}^n$ represents Sutton's eligibility trace \cite{Sutton98}, $\bb_k\in \mathbb{R}^n$ refers to a statistical estimate of the single period reward, and $\bA_k\in \mbb{R}^{n\times n}$ is a sample estimate of the matrix formed by $\bz_{k}(\bpsi_{\theta_k}'(\bx_{k+1},u_{k+1})-\bpsi_{\theta_k}'(\bx_k,u_k))$, which can be viewed as a sample observation of the scaled difference of the observation of the state incidence vector for iterations $k$ and $k+1$, scaled to the feature space by the basis functions. 

\medskip

\noindent{\bf LSTD Actor-Critic for SSP}

\noindent{\bf Initialization:}

Set all entries in $\bz_0, \bA_0, \bb_0$ and $\br_0$ to zeros. Let $\btheta_0$ take some initial value, potentially
corresponding to a heuristic policy.

\noindent{\bf Critic:}
\begin{equation}\label{E:critic}
\begin{array}{c}
\bz_{k+1} = \lambda \bz_k+\bpsi_{\btheta_k}(\bx_k,u_k),\\
\bb_{k+1}=\bb_k+\displaystyle \gamma_k \left[g(\bx_k,u_k) \bz_{k}-\bb_k\right],\\  
\bA_{k+1}=\bA_k + \displaystyle \gamma_k [\bz_{k}(\bpsi_{\theta_k}'
  (\bx_{k+1},u_{k+1})-\bpsi_{\theta_k}'(\bx_k,u_k))\\
  -\bA_k],
\end{array}
\end{equation}
where $\lambda\in[0,1)$, 
  $\gamma_k\stackrel{\triangle}{=}\displaystyle\frac{1}{k}$, and finally 
\begin{equation} \label{critic-r}
\br_{k+1} = -\bA_k^{-1}\bb_k.
\end{equation}
\medskip

\noindent{\bf Actor:}
\begin{equation}\label{E:actor}
\btheta_{k+1}=\btheta_k-\beta_k\Gamma(\br_k)\br_{k}'\bpsi_{\theta_k}(\bx_{k+1},u_{k+1})\bpsi_{\theta_k}(\bx_{k+1},u_{k+1}).
\end{equation}
In the above, $\{\gamma_k\}$ controls the critic step-size, while
$\{\beta_k\}$ and $\Gamma(\br)$ control the actor step-size together. An
implementation of this algorithm needs to make these choices. The role
of $\Gamma(\br)$ is mainly to keep the actor updates bounded, and we can
for instance use
\[
\Gamma(\br)=\left\{\begin{array}{ll}
\displaystyle\frac{D}{||\br||}, &\text{if $||\br||>D$,}\\ 1, &\text{otherwise,}
\end{array}\right.
\]
for some $D>0$. $\{\beta_k\}$ is a deterministic and non-increasing sequence for which we need to have
\begin{equation}\label{E:kbeta}
\sum_k \beta_k=\infty,\quad \sum_k \beta_k^2 < \infty,\quad \lim_{k\to\infty}\frac{\beta_k}{\gamma_k}=0.
\end{equation}
An example of $\{\beta_k\}$ satisfying Eq. (\ref{E:kbeta}) is  
\begin{equation}
\label{eq:examplebetak}
\beta_k=\frac{c}{k\thinspace\ln k},\quad k>1,
\end{equation}
where $c>0$ is a constant parameter. Also, $\bpsi_{\btheta}(\bx,u)$ is defined as
$$\bpsi_{\btheta}(\bx,u)=\nabla_{\btheta} \ln\mu_{\btheta}(u|\bx),$$ where
$\bpsi_{\btheta}(\bx,u) =\bzero$ when $\bx,u$ are such that
$\mu_{\btheta}(u|\bx)\equiv 0$ for all $\btheta$. It is assumed that
$\bpsi_{\btheta}(\bx,u)$ is bounded and continuously differentiable. Note that $\bpsi_{\btheta}(\bx,u)=(\psi_{\btheta}^1(\bx,u),\ldots,\psi_{\btheta}^n(\bx,u))$ where $n$ is the dimensionality of $\btheta$. The convergence of the algorithm is stated in the following Theorem (see the {\bf Appendix} for the proof).

\begin{thm}\label{T:actor}[Actor Convergence] For the LSTD actor-critic
  with some step-size sequence $\{\beta_k\}$ satisfying (\ref{E:kbeta}), for any $\epsilon>0$, there exists some $\lambda$ sufficiently close to $1$, such that
  $\lim\inf_k||\nabla\bar{\alpha}(\btheta_k)||<\epsilon$ w.p.1. That is,
  $\btheta_k$ visits an arbitrary neighborhood of a stationary point
  infinitely often.
\end{thm}

\section{The MRP and its conversion into an SSP problem}
\label{sec:conversion}

In the MRP problem, we assume that there is a set of {\em unsafe} states which are set to be absorbing on the MDP ({\em i.e.}, there is only one control at each state, corresponding to a self-transition with probability $1$). Let $\mathbb{X}_G$ and $\mathbb{X}_U$ denote the set of goal states and unsafe states, respectively. A {\em safe} state is a state that is not unsafe. It is assumed that if the system is at a safe state, then there is at least one sequence of actions that can reach one of the states in $\mathbb X_{G}$ with positive probability. Note that this implies that Assumption \ref{C:FirstProperAss} holds. In the MRP, the goal is to find the optimal policy that maximizes the probability of reaching a state in $\mathbb X_{G}$ from a given initial state. Note that since the unsafe states are absorbing, to satisfy this specification the system must not visit the unsafe states. 

We now convert the MRP problem into an SSP problem, which requires us to change the original MDP (now denoted as MDP$_{\mathbb M}$) into a SSP MDP (denoted as MDP$_{\mathbb S}$).  Note that \cite{alterovitz2007stochastic} established the equivalence between an MRP problem and an SSP problem where the expected reward is maximized.  Here we present a different transformation where an MRP problem is converted to a more standard SSP problem where the expected cost is minimized. 

%Note that such transformation is known (as in \cite{alterovitz2007stochastic}), but here we present a different version as the one in \cite{alterovitz2007stochastic}, since \cite{alterovitz2007stochastic} transformed the problem into an SSP where rewards are maximized.  Here we convert the problem into a standard SSP, where costs are minimized.
%, instead of costs, are being maximized, instead of minimized. 

To begin, we denote the state space of MDP$_{\mathbb M}$ by $\mathbb X_{\mathbb M}$, and define $\mathbb{X}_{\mathbb S}$, the state space of MDP$_{\mathbb S}$, to be 
$$
\mathbb{X}_{\mathbb S}= (\mathbb{X}_{\mathbb M} \setminus \mathbb{X}_G) \cup \{\bx^*\},
$$
where $\bx^*$ denotes a special termination state. Let $\bx_0$ denote the initial state, and $\mathbb{U}$ denote the action space of MDP$_{\mathbb M}$.  We define the action space of MDP$_{\mathbb S}$ to be $\mathbb U$, {\it i.e.}, the same as for MDP$_{\mathbb M}$.

Let $p_{\mathbb M}(\bj|\bx,u)$ denote the probability of transition to state $\bj\in \mathbb X_{\mathbb M}$ if action $u$ is taken at state $\bx\in \mathbb X_{\mathbb M}$.  We now define the transition probability $p_{\mathbb S}(\bj|\bx,u)$ for all states $\bx,\bj\in \mathbb X_{\mathbb S}$ as:
%denote the probability of transition to state $\bj$ if action $u$ is taken at state $\bx\in \mathbb X_{\mathbb S}$. We define $p_{\mathbb S}(\bj|\bx,u)$ for MDP$_{\mathbb S}$ to be:
\begin{eqnarray}
\label{E:tranprob}
p_{\mathbb S}(\bj|\bx,u)=\left\{\begin{array}{ll}
\displaystyle\sum_{\bi\in\mathbb{X}_G} p_{\mathbb M}(\bi|\bx,u), &\text{if } \bj=\bx^*,\\
p_{\mathbb M}(\bj|\bx,u), &\text{if } \bj\in \mathbb{X}_{\mathbb M}
\setminus \mathbb{X}_G, 
\end{array}\right.
\end{eqnarray}
for all $\bx\in \mathbb{X}_{\mathbb M} \setminus (\mathbb{X}_G\cup \mathbb X_{U})$ and all $u\in \mbb{U}$.  Furthermore, we set $p_{\mathbb S}(\bx^{*}|\bx^{*},u)=1$ and $p_{\mathbb S}(\bx_{0}|\bx,u)=1$ if $\bx\in \mathbb X_{U}$, for all $u\in \mathbb U$.  The transition probability of MDP$_{\mathbb S}$ is defined to be the same as for MDP$_{\mathbb M}$, except that the probability of visiting the goal states in MDP$_{\mathbb M}$ is changed into the probability of visiting the termination state; and the unsafe states transit to the initial state with probability $1$.

%the probability of visiting the unsafe states in MDP$_{\mathbb M}$ is changed into the probability of going back to the initial state; and 

%Let $\tilde g(\bj|\bx,u)$ denote the cost if action $u\in \mbb{U}$ is taken in state $\bx$ of the MDP$_{\mathbb S}$, and state $\bj$ is visited.  We define $\tilde g(\bj|\bx,u)=1$ if $\bj$. 

For all $\bx\in \mathbb X_{S}$, we define the cost $g(\bx,u)=1$ if $\bx \in \mathbb{X}_U$, and $g(\bx,u)=0$ otherwise. Define the expected total cost of a policy $\mu$ to be $\bar{\alpha}_{\mu}^{\mathbb S}=\lim_{t\to\infty} E\{\sum_{k=0}^{t-1} g(\bx_k,u_k)|\bx_0\}$ where actions $u_k$ are obtained according to policy $\mu$ in MDP$_{\mathbb S}$.   Moreover, for each policy $\mu$ on MDP$_{\mathbb S}$, we can define a policy on MDP$_{\mathbb M}$ to be the same as $\mu$ for all states $\bx\in \mathbb X_{\mathbb M}\setminus (\mathbb X_{G}\cup \mathbb X_{U})$.  Since actions are irrelevant at the goal and unsafe states in both MDPs, with slight abuse of notation we denote both policies to be $\mu$.  Finally, we define the \emph{Reachability Probability} $R^{\mathbb M}_{\mu}$ as the probability of reaching one of the goal states from $\bx_{0}$ under policy $\mu$ on MDP$_{\mathbb M}$.  The Lemma below relates $R^{\mathbb M}_{\mu}$ and $\bar\alpha^{\mathbb S}_{\mu}$:

\begin{lemma} For any RSP $\mu$, we have ${\text{R}}^{\mathbb M}_{\mu}=\frac{1}{\bar{\alpha}^{\mathbb S}_{\mu}+1}$.
\end{lemma}
\begin{proof} From the definition of the $g(\bx,u)$, $\bar{\alpha}_{\mu}^{\mathbb S}$ is the expected number of times when unsafe states in $\mathbb X_{U}$ are visited before $\bx^{*}$ is reached.  From the construction of MDP$_{\mathbb S}$, reaching $\bx^*$ in MDP$_{\mathbb S}$ is equivalent to reaching one of the goal states in MDP$_{\mathbb M}$.   On the other hand, for MDP$_{\mathbb M}$, by definition of $\mathbb X_{G}$ and $\mathbb X_{U}$, in the Markov chain generated by $\mu$, the states $\mathbb X_{G}$ and $\mathbb X_{U}$ are the only absorbing states, and all other states are transient.  Thus, the probability of visiting a state in $\mathbb X_{U}$ from $\bx_{0}$ on MDP$_{\mathbb M}$ is $1-R_{\mu}^{\mathbb M}$, which is the same as the probability of visiting $\mathbb X_{U}$ for each run of MDP$_{\mathbb S}$, due to the construction of transition probabilities \eqref{E:tranprob}.  We can now consider a geometric distribution where the probability of success is $R_{\mu}^{\mathbb M}$.  Because $\bar{\alpha}_{\mu}^{\mathbb S}$ is the expected number of times when an unsafe state in $\mathbb X_{U}$ is visited before $\bx^{*}$ is reached, this is the same as the expected number of failures of Bernoulli trails (with probability of success being $R_{\mu}^{\mathbb M}$) before a success. This implies $\bar{\alpha}_{\mu}^{\mathbb S}=\frac{1-R_{\mu}^{\mathbb M}}{R_{\mu}^{\mathbb M}}$.  Rearranging $\bar{\alpha}_{\mu}^{\mathbb S}=\frac{1-R_{\mu}^{\mathbb M}}{R_{\mu}^{\mathbb M}}$ completes the proof.
%Since each transition is generated independent of the previous transitions, the process can be seen as a sequence of independent Bernoulli trials until the first success (reaching the terminal state) is observed (a failure in these Bernoulli trails corresponds to a return to the initial state), hence the number of failures before observing the first success follows a Geometric distribution with mean equal to inverse of the probability of success, {\em i.e.}, $\bar{\alpha}_{\mu}=\frac{1}{{\text{R}}_{\mu}}$.
\end{proof}

The above lemma means that $\mu$ as a solution to the SSP problem on MDP$_{\mathbb S}$ (minimizing $\bar{\alpha}_{\mu}^{\mathbb S}$) corresponds to a solution for the MRP problem on MDP$_{\mathbb M}$ (maximizing $R_{\mu}^{\mathbb M}$). Note that the algorithm uses a sequence of simulated trajectories, each of which starting at $\bx_0$ and ending as soon as $\bx^*$ is visited for the first time in the sequence. Once a trajectory is completed, the state of the system is reset to the initial state $\bx_0$ and the process is repeated. Thus, the actor-critic algorithm is applied to a modified version of the MDP$_{\mathbb S}$ where transition to a goal state is always followed by a transition to the initial state.

%Note that transformation of MDP$_{\mathbb M}$ to MDP$_{\mathbb S}$ establishes that if we use actor-critic to minimize total cost of MDP$_{\mathbb S}$, it is equal to as if we are maximizing the reachability probability of MDP$_{\mathbb M}$ (Theorem above). The actor-critic itself is a simulation-based algorithm that uses simulated trajectories which could be created in real time; the input to the algorithm at each iteration is only the state just observed, hence the algorithm does not requires the knowledge of transition probabilities.

\section{Case study}
\label{S:case}
In this section we apply our algorithm to control a robot moving in a square-shaped mission environment, which is partitioned into 2500 smaller square regions (a $50\times 50$ grid) as shown in Fig.~\ref{robotEnvironment}. We model the motion of the robot in the environment as an MDP: each region corresponds to a state of the MDP, and in each region (state), the robot can take the following control primitives (actions): ``North'', ``East'', ``South'', ``West'', which represent the directions in which the robot intends to move (depending on the location of a region, some of these actions may not be enabled, for example, in the lower-left corner, only actions ``North'' and ``East'' are enabled).  These control primitives are not reliable and are subject to noise in actuation and possible surface roughness in the environment. Thus, for each motion primitive at a region, there is a probability that the robot enters an adjacent region.

\begin{figure}[ht]
\centering
\includegraphics[width=0.8\columnwidth]{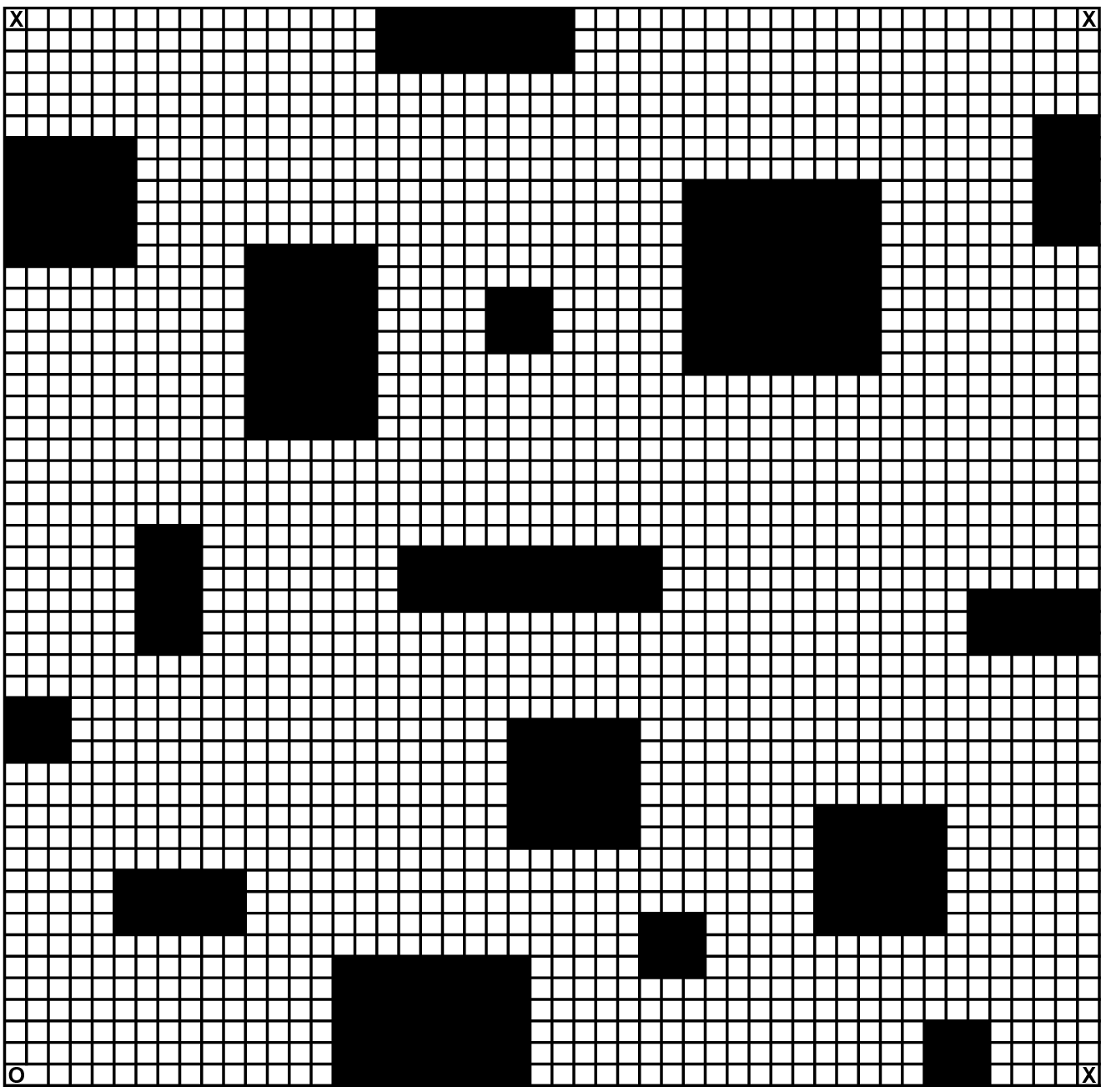}
\caption{View of the mission environment. The initial region is marked by o, the goal regions by x, and the unsafe regions are shown in black.}
\label{robotEnvironment}
\end{figure}

We label the region in the south-west corner as the initial state. We marked the regions located at the other three corners as the set of \emph{goal} states as shown in Fig.~\ref{robotEnvironment}. We assume that there is a set of {\em unsafe} states $\mathbb X_{U}$ in the environment (shown in black in Fig. \ref{robotEnvironment}). Our goal is to find the optimal policy that maximizes the probability of reaching a state in $\mathbb X_{G}$ (set of goal states) from the initial state (an instance of an MRP problem). 

\subsection{Designing an RSP}

To apply the LSTD Actor-Critic algorithm, the key step is to design an
RSP  $\mu_{\btheta}(u|\bx)$.  In this case study, we define the RSP to be an
exponential function of two scalar parameters $\theta_1$ and $\theta_2$,
respectively. These parameters are used to provide a balance between {\em safety} 
and {\em progress} from applying the control policy.

% For each state $\bx$, the safety score $s(\bx)$ is defined as the portion of safe states with
% respect to the total neighboring states.  Instead of using Euclidean distance, we
For each pair of states $\bx_{i},\bx_{j}\in \mathbb X$, we define $d(\bx_{i},\bx_{j})$ as the minimum number of transitions from $\bx_{i}$ and $\bx_{j}$.  We denote $\bx_{j}\in N(\bx_{i})$ if and only if $d(\bx_i,\bx_j) \leq r_n $, where $r_n$ is a fixed integer given apriori.  If $\bx_{j}\in N(\bx_{i})$, then we say $\bx_{i}$ is in the neighborhood of $\bx_{j}$, and $r_{n}$ represents the radius of the neighborhood around each state. 
 
For each state $\bx\in \mathbb X$, the safety score $s(\bx)$ is defined as the ratio of the safe neighbouring states over all neighboring states of $\bx$.  To be more specific, we define
\begin{equation}
\label{eq:safety}
s(\bx) = \frac{\sum_{\by
\in N(\bx)}I_{s}(\by)}{|N(\bx)|} 
\end{equation}
where $I_{s}(\by)$ is an indicator function such that $I_{s}(\by)=1$ if and only if $\by\in \mathbb X\setminus \mathbb X_{U}$ and $I_{s}(\by)=0$ if otherwise.
A higher safety score for the current state of robot means it is less likely for the robot to reach an unsafe region in the future.  
 
We define the progress score of a state $\bx\in \mathbb X$ as $d_g(\bx):=\min_{\by\in \mathbb X_{G}}d(\bx, \by)$, which is the minimum number of transitions from $\bx$ to any goal region.   We can now propose the RSP policy, which is a Boltzmann policy as defined in \eqref{eq:boltzmann}. Note that $\mathbb U=\{u_{1},u_{2},u_{3},u_{4}\}$, which corresponds to ``North'', ``East'', ``South'', and ``West'', respectively.  
We first define
\begin{equation}
\begin{split}
a_i (\btheta)= F_i(\bx)e^{\theta_1 E\{s(f(\bx,u_i))\}+ \theta_2 E\{d_{g}(f(\bx,u_i))-d_{g}(\bx)\}},
\end{split}
\end{equation}
where $\btheta:=(\theta_{1},\theta_{2})$, and $F_i(\bx)$ is an indicator function such that $F_{i}(\bx)=1$ if $u_{i}$ is available at $\bx_{i}$ and $F_{i}(\bx)=0$ if otherwise.  Note that the availability of control actions at a state is limited for the states at the boundary.  For example, at the initial state, which is at the lower-left corner, the set of available actions is $\{u_{1},u_{2}\}$, corresponding to ``North'' and ``East'', respectively.  If an action $u_{i}$ is not available at state $\bx$, we set $a_{i}(\btheta)=0$, which means that $\mu_{\btheta}(u_{i}|\bx)=0$.

Note that $a_{i}(\btheta)$ is defined to be the combination of the expected safety score of the next state applying control $u_{i}$, and the expected improved progress score from the current state applying $u_{i}$, weighted by $\theta_{1}$ and $\theta_{2}$.  The RSP is then given by
\begin{equation}
 \label{eq:rsp}
\bmu_{\btheta}(u_{i}|\bx) = \frac{a_{i}(\btheta)}{\sum_{i=1}^4 a_i(\btheta)}.
\end{equation}
We note that Assumption \ref{C:SecondProperAss} holds for the
 proposed RSP.  Moreover, Assumption \ref{C:FirstProperAss} also holds, therefore Theorem \ref{T:mainThmm}
 holds for this RSP.

\subsection{Generating transition probabilities}
To implement the LSTD Actor-Critic algorithm, we first constructed the MDP.  As mentioned above, this MDP represents the motion of the robot in the environment where each state corresponds to a cell in the environment (Fig. \ref{robotEnvironment}).  To capture the transition probabilities of the robot from a cell to its adjacent one under an action, we built a simulator.

The simulator uses a unicycle model (see, {\it e.g.}, \cite{lavalle2006planning}) for the dynamics of the robot with noisy sensors and actuators. In this model, the motion of the robot is determined by specifying a forward and an angular velocity.  At a given region, the robot implements one of the following four controllers (motion primitives) - ``East'', ``North'', ``West'', ``South''.  Each of these controllers operates by obtaining the difference between the current heading angle and the desired heading angle. Then, it is translated into a proportional feedback control law for angular velocity.  The desired heading angles for the ``East'', ``North'', ``West'', and ``South'' controllers are $0^\circ$, $90^\circ$, $180^\circ$, and $270^\circ$, respectively. Each controller also uses a constant forward velocity. The environment in the simulator is a 50 by 50 square grid as shown in Fig. \ref{robotEnvironment}. To each cell of the environment, we randomly assigned a surface roughness which affects the motion of the robot in that cell.  The perimeter of the environment is made of walls, and when the robot runs to them, it bounces with the mirror-angle.

To find the transition probabilities, we performed a total of 5000 simulations for each controller and state of the MDP.  In each trial, the robot was initialized at the center of the cell, and then an action was applied.  The robot moved in that cell according to its dynamics and surface roughness of the region.  As soon as the robot exited the cell, a transition was encountered.  Then, a reliable center-converging controller was automatically applied to steer the robot to the center of the new cell.  We assumed that the center-converging controller is reliable enough that always drives the robot to the center of the new cell before exiting it. Thus, the robot always started from the center of a cell.  This makes the process Markov (the probability of the current transition depends only the control and the current state, and not on the history up to the current state).  We also assumed perfect observation at the boundaries of the cells.  %This allowed us to capture transitions perfectly.

It should be noted that, in general, it is not required to have all the transition probabilities of the model in order to apply the LSTD Actor-Critic algorithm, but rather, we only need transition probabilities along the trajectories of the system simulated while running the algorithm.  This becomes an important advantage in the case where the environment is large and obtaining all transition probabilities becomes infeasible. % For such scenarios, simulators, such as the one described above, can be utilized to generate the trajectories (paths) of the robot required by the Actor-Critic algorithm.

\subsection{Results}
We first obtained the exact optimal policy for this problem using the methods described in \cite{LaWaAnBe-ICRA10,Ding11}. 
The maximal reachability probability is 99.9988\%. We then used our LSTD actor-critic algorithm to optimize with respect to $\btheta$ as outlined in Sec. \ref{sec:actorCriticAlgo} and \ref{sec:conversion}.

Given $\btheta$, we can compute the exact probability of reaching $\mathbb X_{G}$ from any state $\bx\in \mathbb X$ applying the RSP $\mu_{\btheta}$ by solving the following set of linear equations
\begin{align}
\label{eq:RSPexactProb}
&&p_{\btheta}(\bx)=\sum_{u\in \mathbb U} \mu_{\btheta} (u|\bx)\sum_{\by\in \mathbb X} p(\by | \bx, u) p_{\btheta}(\by),\nonumber\\
 && \textrm{for all } x\in \mathbb X\setminus(\mathbb X_{U}\cup \mathbb X_{G})
\end{align}
such that $p_{\btheta}(\bx)=0$ if $\bx\in \mathbb X_{U}$ and $p_{\btheta}(\bx)=1$ if $\bx\in \mathbb X_{G}$.  Note that the equation system given by \eqref{eq:RSPexactProb} contains exactly $|\mathbb X|-|\mathbb X_{U}|-|\mathbb X_{G}|$ number of equations and unknowns.

We plotted in Fig.~\ref{small2} the reachability probability of the RSP from the initial state (i.e., $p_{\btheta}(\bx_{0})$) against the number of iterations in the actor-critical algorithm each time $\btheta$ is updated.  
As $\btheta$ converges, the reachability probability converges to 90.3\%.  The parameters for this examples are: $r_n = 2$, $\lambda=0.9$, $D=5$ and the initial $\btheta$ is $(50, -10)$.  We use \eqref{eq:examplebetak} for $\beta_{k}$ with $c=0.05$.

\begin{figure}[t]
\begin{center}
\includegraphics[scale=.58]{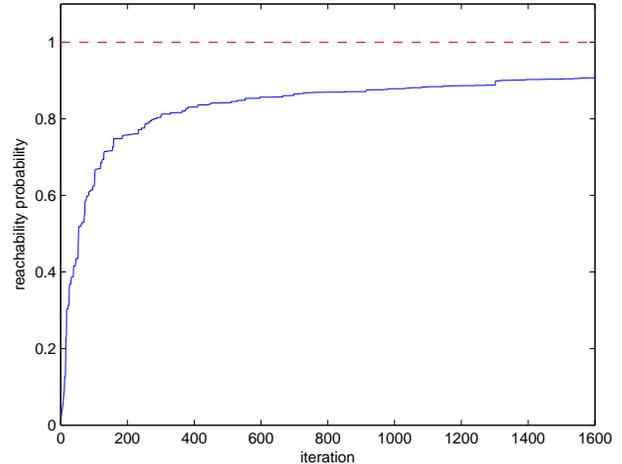}
\end{center}
\caption{The dashed line represents the optimal solution (the maximal reachability probability) and the solid line represents the exact reachability probability for the RSP as a function of the number of iterations applying the proposed algorithm.}
 \label{small2}
\end{figure}
\section{Conclusion}\label{S:conclusion}
We considered the problem of finding a control policy for a Markov Decision Process (MDP) to maximize the probability of reaching some states of the MDP while avoiding some other states. We presented a transformation of the problem into a Stochastic Shortest Path (SSP) MDP and developed a new approximate dynamic programming algorithm to solve this class of problems. The algorithm operates on a sample-path of the system and optimizes the policy within a pre-specified class parameterized by a parsimonious set of parameters. Simulation results confirm the effectiveness of the proposed solution in robot motion planning applications.

\section*{Appendix: Convergence of the LSTD Actor-Critic Algorithm}
We first cite the theory of {\em linear stochastic approximation driven
  by a slowly varying Markov chain} \cite{Konda} (with
simplifications).

Let $\{\by_k\}$ be a finite Markov chain whose transition
probabilities depend on a parameter $\btheta\in\mathbb{R}^n$.
Consider a generic iteration of the form
\begin{equation}\label{E:giterate}
\bs_{k+1}=\bs_k+\gamma_k (\bh_{\btheta_k}(\by_{k+1})-
\bG_{\btheta_k}(\by_{k+1})\bs_k)+\gamma_k\bXi_{k}\bs_k, 
\end{equation}
where $\bs_k\in\mathbb{R}^m$, and $\bh_{\btheta}(\cdot)\in \mbb{R}^m,
\bG_{\btheta}(\cdot)\in \mbb{R}^{m\times m}$ are $\btheta$-parameterized
vector and matrix functions, respectively. It has been shown in
\cite{Konda} that the critic in (\ref{E:giterate}) converges if the
following set of conditions are met.

\begin{cond}\label{C:slowV}
\begin{enumerate}
\item The sequence $\{\gamma_k\}$ is deterministic, non-increasing, and
$$\sum_k \gamma_k=\infty,\quad \sum_k \gamma_k^2 < \infty.$$
\item The random sequence $\{\btheta_k\}$ satisfies
  $||\btheta_{k+1}-\btheta_k||\le\beta_kH_k$ for some process $\{H_k\}$ with
  bounded moments, where $\{\beta_k\}$ is a deterministic sequence such that
$$\sum_k \bigg(\frac{\beta_k}{\gamma_k}\bigg)^d <\infty\quad\text{for some }d>0.$$
\item $\bXi_k$ is an $m\times m$-matrix valued martingale difference
  with bounded moments. 
\item For each ${\btheta}$, there exist $\bar{\bh}({\btheta})\in\mathbb{R}^m$,
  $\bar{\bG}({\btheta})\in\mathbb{R}^{m \times m}$, and corresponding
  $m$-vector and $m\times m$-matrix functions 
  $\hat{\bh}_{\btheta}(\cdot)$, $\hat{\bG}_{\btheta}(\cdot)$ that
  satisfy the 
  Poisson equation. That is, for each $\by$,
$$\hat{\bh}_{\btheta}(\by)= \bh_{\btheta}(\by) -
  \bar{\bh}({\btheta})+(P_{\btheta}\hat{\bh}_{\btheta})(\by),$$ 
$$\hat{\bG}_{\btheta}(\by)=\bG_{\btheta}(\by) -
  \bar{\bG}({\btheta})+(P_{\btheta}\hat{\bG}_{\btheta})(\by).$$ 
\item For some constant $C$ and for all ${\btheta}$, we have
$\max (|| \bar{\bh}({\btheta})|| , ||\bar{\bG}({\btheta})||)\leq C.$
\item For any $d>0$, there exists $C_d>0$ such that
$\sup_k \mathbf{E}[|| \bbf_{{\btheta}_k}(\by_k)||^d]\leq C_d,$
where $\bbf_{\btheta}(\cdot)$ represents any of the functions
$\hat{\bh}_{\btheta}(\cdot)$, $\bh_{\btheta}(\cdot)$, $\hat{\bG}_{\btheta}(\cdot)$
and $\bG_{\btheta}(\cdot)$. 
\item For some constant $C>0$ and for all
  ${\btheta},\bar{{\btheta}}\in\mathbb{R}^n$,  
$\max (|| \bar{\bh}({\btheta})-\bar{\bh}(\bar{{\btheta}})|| ,
  ||\bar{\bG}({\btheta})-\bar{\bG}(\bar{{\btheta}})||)\leq C|| {\btheta}
  - \bar{{\btheta}}||.$ 
\item There exists a positive measurable function $C(\cdot)$ such
  that for every $d>0$, $\sup_k \mathbf{E}[C(\by_k)^d]<\infty$, and $||
  \bbf_{\btheta}(\by)-\bbf_{\bar{{\btheta}}}(\by)||\leq C(\by)|| {\btheta} -
  \bar{{\btheta}}||.$ 
\item There exists $a>0$ such that for all $\bs\in \mathbb{R}^m$ and
  ${\btheta}\in \mathbb{R}^n$ 
$$\bs'\bar{\bG}({\btheta})\bs\geq a||\bs||^2.$$
\end{enumerate}
\end{cond}
For now, let's focus on
the first two items of Condition~\ref{C:slowV}. Recall that for any
matrix $\bA$, $\bv(\bA)$ is a column vector that stacks all row vectors
of $\bA$ (also written as column vectors).  Simple algebra suggests that
the core iteration of the LSTD critic can be written as
(\ref{E:giterate}) with
\begin{gather}
\bs_k=\left[\begin{array}{c} \bb_k\\ \bv(\bA_k)
    \\1\end{array}\right],\qquad 
\by_k=(\bx_k,u_k,\bz_k), \nonumber \\
\begin{array}{c}
\bh_{\btheta}(\by)=\left[\begin{array}{c}
g(\bx,u)\bz\\\bv(\bz((P_{\btheta}\bpsi_{\btheta}')(\bx,u) -
\bpsi_{\btheta}'(\bx,u)))\\1\end{array}\right],\\ 
\bG_{\btheta}(\by)=\left[\begin{array}{c} \bI \end{array}\right],
\end{array}\label{E:LSTDiterate}\\
\begin{array}{l}
\bXi_k = 
\left[\begin{array}{ccc} \bzero&\bzero&\bzero\\\bzero&\bzero&
   D\\\bzero &\bzero
    &\bzero\end{array}\right],
\end{array} \nonumber
\end{gather}
where 
$$
D=\bv(\bz_k(\bpsi_{\btheta_k}'(\bx_{k+1},u_{k+1})-(P_{\btheta}\bpsi_{\btheta})'(\bx_k,u_k))),
$$
and $M$ is an arbitrary (large) positive constant whose role is to
facilitate the convergence proof, and $\by=(\bx,u,\bz)$ denotes a value
of the triplet $\by_k$.

The step-sizes $\gamma_k$ and $\beta_k$ in (\ref{E:critic}) and
(\ref{E:actor}) correspond 
exactly to the $\gamma_k$ and $\beta_k$ in Condition \ref{C:slowV}.(1)
and \ref{C:slowV}.(2), respectively. If the MDP has finite state and
action space, then the conditions on $\{\beta_k\}$ reduce to
(\cite{Konda}) 
\begin{equation}\label{E:kbetan}
\sum_k \beta_k=\infty,\quad \sum_k \beta_k^2 < \infty,\quad \lim_{k\to\infty}\frac{\beta_k}{\gamma_k}=0,
\end{equation}
where $\{\beta_k\}$ is a deterministic and non-increasing sequence. Note that we can use
$\gamma_k=1/k$ (cf. Condition \ref{C:slowV}). 
The following theorem establishes
the convergence of the critic.

\begin{thm}\label{T:critic}[Critic Convergence] For the LSTD
  actor-critic (\ref{E:critic}) and (\ref{critic-r}) with some step-size
  sequence $\{\beta_k\}$ satisfying (\ref{E:kbetan}), the sequence $\bs_k$ is bounded, and 
\begin{equation}\label{E:gesti}
\lim_{k\to\infty}|\bar{\bG}(\btheta_k)\bs_k - \bar{\bh}(\btheta_k)|=0.
\end{equation}
\end{thm}
\begin{proof}
To show that (\ref{E:giterate}) converges with
$\bs,\by,\bh_{\btheta}(\cdot),\bG_{\btheta}(\cdot)$ and $\bXi$
substituted by (\ref{E:LSTDiterate}), the conditions
\ref{C:slowV}.(1)-(9) should be checked. However, a
comparison with the convergence proof for the TD($\lambda$) critic in
\cite{Konda03} gives a simpler
proof. Let $$\bF_{\btheta}(\by)=\bz(\bpsi_{\btheta}'(\bx,u)-
(P_{\btheta}\bpsi_{\btheta})'(\bx,u)).$$ While proving the convergence
of TD($\lambda$) critic operating concurrently with the actor,
\cite{Konda03} showed that
\[
\tilde{\bh}_{\btheta}(\by)=\left[\begin{array}{c}
    \tilde{h}^{(1)}_{\btheta}(\by)\\\tilde{\bh}^{(2)}_{\btheta}(\by)\end{array}\right]=\left[\begin{array}{c}
    Mg(\bx,u)\\g(\bx,u)\bz\end{array}\right],
    \]
    \[
\tilde{\bG}_{\btheta}(\by)=\left[\begin{array}{cc} 1 & \bzero\\\bz/M &
    \bF_{\btheta}(\by)\end{array}\right],
\]
and
$$\tilde{\bXi}_k=\left[\begin{array}{cc} \bzero & \bzero\\\bzero &
    \bz_k(\bpsi_{\btheta_k}'(\bx_{k+1},u_{k+1})-
    (P_{\btheta}\bpsi_{\btheta})'(\bx_k,u_k))\end{array}\right]$$  
satisfy Condition \ref{C:slowV}.(3)-\ref{C:slowV}(8). In our case,
(\ref{E:LSTDiterate}) 
can be rewritten as
\begin{equation}
\bh_{\btheta}(\by)=\left[\begin{array}{c}\tilde{\bh}^{(2)}_{\btheta}(\by)\\-
    \bF_{\btheta}(\by)\\1\end{array}\right],\;  
\bG_{\btheta}(\by)=\left[\begin{array}{c}\bI\end{array}\right],\; 
\bXi_k=\left[\begin{array}{c}
    \tilde{\bXi}_k\\\bzero\end{array}\right].
\end{equation}
Note that although the
two iterates are very different, they involve the same quantities and
both in a linear fashion. So,
$\bh_{\btheta}(\cdot),\bG_{\btheta}(\cdot)$ and $\bXi_k$ also satisfy
conditions \ref{C:slowV}.(3)-\ref{C:slowV}(8). Meanwhile, the step-size
$\{\gamma_k\}$ satisfies condition \ref{C:slowV}.(1), and the step-size
$\{\beta_k\}$ satisfies Eq. (\ref{E:kbetan}) (which is as explained above
implies condition \ref{C:slowV}.(2)). Now, only condition (9) remains to
be checked. To that end, note that all diagonal elements of
$\bG_{\btheta}(\by)$ equal to one, so, $\bG_{\btheta}(\by)$ is positive
definite. This proves the
convergence. Using the same correspondence and the result in
\cite{Konda03}, one can further check that (\ref{E:gesti}) also
holds here. 
\end{proof}

\subsection*{Proof of Theorem~\ref{T:actor}:} The result follows by setting $\bphi_{\btheta}=\bpsi_{\btheta}$ and
following the proof in Section~6 of \cite{Konda03}. \qed

\bibliographystyle{IEEEtran}

\bibliography{/home/yannisp/Private/bib/abbrev,/home/yannisp/Private/bib/IEEEabrv,/home/yannisp/Private/bib/communications,/home/yannisp/Private/bib/my,/home/yannisp/Private/bib/optimization,/home/yannisp/Private/bib/stochastics,actor-critic,Papers}
}

\end{document}